\documentclass{article}
\usepackage{nips14submit_e,times}
\usepackage{url}
\definecolor{mydarkblue}{rgb}{0,0.08,0.45}

\usepackage[utf8]{inputenc}
\usepackage{amssymb}
\usepackage{latexsym}
\usepackage{amsbsy}
\usepackage{amsmath}
\usepackage{amsthm}
\usepackage{enumerate}
\usepackage{graphicx}
\usepackage{subfigure}
\usepackage[numbers]{natbib}
\usepackage{xspace}

\newcommand{\Hcal}{\mathcal{H}}

\newcommand{\posterior}{\rho}
\newcommand{\prior}{\pi}

\newcommand{\xbf}{x}
\newcommand{\xb}{\xbf}
\newcommand{\ex}{(\xbf,y)}
\newcommand{\Y}{Y}

\newcommand{\I}{\mathrm{\bf I}}

\newcommand{\sign}{\operatorname{sign}}

\newcommand{\R}{\mathbb{R}}
\newcommand{\Risk}{\mathbf{R}}

\newcommand{\RP}{\Risk_{D}}

\newcommand{\BQ}{B_{\posterior}}

\DeclareMathOperator*{\Prob}{\mathrm{\mathbf{Pr}}}
\DeclareMathOperator*{\Esp}{\mathrm{\mathbf{E}}}
\DeclareMathOperator*{\Var}{\mathrm{\mathbf{Var}}}
\DeclareMathOperator*{\argmin}{\mathrm{argmin}}
\DeclareMathOperator*{\argmax}{\mathrm{argmax}}
\newcommand{\eqdef}{=}

\newcommand{\gc}[2]{\Esp_{h\sim \posterior}\I(h(#1)=#2)}
\newcommand{\momentone}{\mu_1}
\newcommand{\momenttwo}{\mu_2}
\newcommand{\MQ}{M_\posterior}
\newcommand{\MQP}{M_\posterior^D}

\newcommand{\SQ}{S_{\posterior, c}}
\newcommand{\SQP}{S_{\posterior,c}^D}
\newcommand{\OQ}{M_{\posterior,\omega}}
\newcommand{\OQtwo}{M_{\posterior,2}}
\newcommand{\OQQ}{M_{\posterior,Q}}
\newcommand{\OQP}{M_{\posterior,\omega}^D}

\newcommand{\Cbound}{\mbox{$\mathcal{C}$-bound}\xspace}
\newcommand{\Cbounds}{\mbox{$\mathcal{C}$-bounds}\xspace}

\newtheorem{definition}{Definition}

\newtheorem{theorem}{Theorem}

\nipsfinalcopy

\title{On Generalizing the C-Bound to the Multiclass and Multi-label~Settings}

\author{
Fran\c cois Laviolette\\
IFT-GLO, Universit\'e Laval\\
Qu\'ebec (QC), Canada \\
\texttt{\small francois.laviolette@ift.ulaval.ca} 
\And
Emilie Morvant \\
Univ. de St-Etienne, France\\
LaHC, UMR CNRS 5516\\
\texttt{\small emilie.morvant@univ-st-etienne.fr}
\And
Liva Ralaivola \\
Aix-Marseille Universit\'e\\
LIF, UMR CNRS 7279\\
\texttt{\small liva.ralaivola@lif.univ-mrs.fr} \\
\And
Jean-Francis Roy \\
IFT-GLO, Universit\'e Laval\\
Qu\'ebec (QC), Canada \\
\texttt{\small jean-francis.roy@ift.ulaval.ca} 
}

\nipsfinalcopy 

\begin{document}

\maketitle

\begin{abstract}
	The \Cbound, introduced in Lacasse et al.~\cite{Lacasse07}, gives a tight upper bound on the risk of a binary majority vote classifier. In this work, we present a first step towards extending this work to more complex outputs, by providing generalizations of the \Cbound to the multiclass and multi-label settings.
\end{abstract}

\section{Introduction}

In binary classification, many state-of-the-art algorithms output prediction functions that can be seen as a majority vote of ``simple'' classifiers. Ensemble methods such as Bagging \citep{b-96}, Boosting \citep{schapire99} and Random Forests \citep{randomforests} are well-known examples
of  learning algorithms that output majority votes. Majority votes are also central in the Bayesian approach~(see Gelman et al. \cite{gelman2004bayesian} for an introductory text);   in this setting, the majority vote is generally called the \emph{Bayes Classifier}.  
It is also -interesting to point out that classifiers produced by kernel methods, such as the Support Vector Machine \cite{DBLP:journals/ml/CortesV95}, can also be viewed as majority votes. Indeed, to classify an example~$x$, the SVM classifier computes 
\ $\sign \left(\sum_{i=1}^{|S|} \alpha_i\, y_i\, k(\xb_i,\xb)\right)$.
Hence, as for standard binary majority votes, if the total weight of each \,$\alpha_i \,y_i\,k(\xb_i,\xb)$\, that votes positive is larger than the total weight for the negative choice, the classifier will output a $+1$ label (and a $-1$ label in the opposite case). 

Most bounds on majority votes take into account the margin of the majority vote  on an example $(x,y)$, that is the difference between the total vote weight that has been given to the winning class minus the weight given to the alternative class. As an example, PAC-Bayesian bounds give bounds on majority votes classifiers by relating it to a stochastic classifier, called the \emph{Gibbs} classifier which is, up to a linear transformation equivalent to the first statistical moment of the margin when $(x,y)$ is drawn {\it i.i.d.\ }from a distribution \cite{MinCq}. Unfortunately, in most ensemble methods, the voters are weak and no majority vote can obtain high margins. 
Lacasse et al. \cite{Lacasse07} proposed a tighter relation between the risk of the majority vote that take into account both the first and the second moments of the margin: the \Cbound. This sheds a new light on the behavior of majority votes: it is not only how good are the voters but also how they are correlated in their voting. Namely, this has inspired a new learning algorithm named MinCq \cite{MinCq}, whose performance is state-of-the-art. In this work, we generalize the \Cbound for multiclass and multi-label weighted majority votes as a first step towards the goal of designing learning algorithms for more complex outputs.

This paper is organized as follows.
Section~\ref{sec:cborne} recalls the \Cbound in binary classification. We generalize it to the multiclass and multi-label settings in Sections~\ref{sec:cborne_multi} and~\ref{sec:cborne_multilabel}. We conclude in Section~\ref{sec:conclu}.

\section{The \Cbound for Binary Classification}
\label{sec:cborne}

In this section, we recall the \Cbound \cite{Lacasse07,MinCq} in the binary classification setting.

Let $X\!\subseteq\! \R^d$ be the input space of dimension $d$, and let  $Y = \{-1,+1\}$ be the output space. 
The learning sample $S\!=\!\{(\xbf_i,y_i)\}_{i=1}^{m}$ is constituted by $m$ examples drawn {\it i.i.d.\ }from a fixed but unknown distribution $D$ over $X\!\times\! Y$.
Let $\Hcal$ be a set of real-valued voters from  $X$ to $Y$. 
Given a prior distribution $\prior$ on $\Hcal$ and given  $S$, 
the goal of 
the PAC-Bayesian approach is to find  
the posterior distribution $\posterior$ on $\Hcal$ which minimizes the true risk of the $\posterior$-weighted majority vote $\BQ(\cdot)$ given by\\
\centerline{$\displaystyle \RP(\BQ) =  \Esp_{(\xbf,y)\sim D} \I \left(\BQ(\xbf)\ne y\right)\,, \quad\text{where}\quad \BQ(\xbf)=\sign\left[\Esp_{h\sim \posterior} h(\xbf)\right]\,,$}\\
and where $\I(a) = 1$ if predicate $a$ is true and $0$ otherwise.

It is well-know that minimizing  $\RP(\BQ)$ is NP-hard. 
To get around this problem, one solution is to make use of the \Cbound which is a tight bound over $\RP(\BQ)$.
This bound is based on the notion of margin of $\BQ(\cdot)$ defined as follows.
\begin{definition}[the margin] \label{def:marge}
	Let $\MQP$ be the random variable that, given an example $\ex$ drawn according to $D$, outputs the margin of $\BQ(\cdot)$ on that example, defined by $
	\MQ (\xbf,y) = y \Esp_{h\sim \posterior} h(\xbf). 
	$
\end{definition}
We then consider the first and second statistical moments of the random variable $\MQP$, respectively given by $\momentone(\MQP) = \Esp_{(\xbf,y)\sim D} \MQ(\xbf,y)$ and $\momenttwo(\MQP) = \Esp_{(\xbf,y)\sim D} \left(\MQ(\xbf,y)\right)^2$.

According to the definition of the margin, $\BQ(\cdot)$ correctly classifies an example $\ex$ when its margin is strictly positive, i.e. $\RP(\BQ) = \Prob_{(\xbf,y)\sim D} \left(\MQ (\xbf,y) \leq 0 \right)$. This equality makes it possible to prove the following theorem.

\begin{theorem}[The \Cbound of Laviolette et al. \cite{MinCq}]
	\label{theo:Cbound_bin}
	For every distribution $\posterior$ on a set of real-valued functions  $\Hcal$, and for every distribution $D$ on $X\times Y$, if  $\momentone(\MQP) > 0$, then we have
	\begin{equation*}
	\RP(\BQ)\ \leq\ 1- \frac{\left(\momentone(\MQP)\right)^2}{\momenttwo(\MQP)}\,.
	\end{equation*}
\end{theorem}
\begin{proof}The Cantelli-Chebyshev inequality states that any random variable $Z$ and any $a>0$, we have that $\Prob \left(Z\leq \Esp_{}\left[Z\right] - a\right)\ \leq\ \frac{\Var Z}{\Var Z + a^2}$. We obtain the result by applying this inequality with $Z = \MQ(\xbf,y)$, and with $a = \momentone(\MQP)$, and by using the definition of the variance.
\end{proof}
Note that the minimization of the empirical counterpart of the \Cbound is a natural solution for learning a distribution $\posterior$ that leads to a $\posterior$-weighted majority vote $\BQ(\cdot)$ with low error.
This strategy is justified thanks to an elegant PAC-Bayesian generalization bound over the \Cbound, and have led to a simple learning algorithm called MinCq \cite{MinCq}.

In the following, we generalize this important theoretical result in the PAC-Bayesian literature to the multiclass setting.

\section{Generalizations of the \Cbound for Multiclass Classification}
\label{sec:cborne_multi}
In this section, we stand in the multiclass classification setting where the input space is still $X\! \subseteq\! \R^d$, but the output space is $Y\! =\! \{1,\ldots,Q\}$, with a finite number of classes $Q\!\geq\! 2$.
Let $\Hcal$ be a set of multiclass voters from  $X$ to $Y$. 
We recall that given a prior distribution $\prior$ over $\Hcal$ and given a  sample $S$, {\it i.i.d.\ }from $D$, the PAC-Bayesian approach looks for the $\posterior$ distribution which minimizes the true risk of the 
 majority vote $\BQ(\cdot)$.
In the multiclass classification setting, $\BQ(\cdot)$ is defined by
\begin{align}
\label{eq:bayes_multiclasse}
B_{\posterior}(\xbf) = \argmax_{c\in Y}\left[\Esp_{h\sim \posterior}\I(h(\xbf)=c)\right].
\end{align}

As in binary classification, the risk $\RP(\BQ)$ of a $\posterior$-weighted majority vote can be related to the notion of margin realized on an example $(\xbf,y)$.
However, in multiclass classification, such a notion can be expressed in a variety of manners.
In the next section, we present three versions of multiclass margins that are equivalent in binary classification.

\subsection{Margins in Multiclass Classification}

We first make use of the multiclass margin proposed by Breiman~\cite{randomforests} for the random forests, which can be seen as the usual notion of margin. Note that when $Y=\{-1,+1\}$, we recover the usual notion of binary margin of Definition~\ref{def:marge}.
\begin{definition}[the multiclass margin]\rm
	\label{def:posterior_marge_multiclasse}
	Let $D$ be a distribution over $X\times Y$, let $\Hcal$ be a set of multiclass voters. Given a distribution $\posterior$ on $\Hcal$, the margin of the majority vote $\BQ(\cdot)$ realized on  $(\xbf,y)\!\sim\! P$ is
	\vspace{-1.5mm}
	\begin{align*}
	\MQ(\xbf,y)  \eqdef  \gc{\xbf}{y}  -  \max_{c\in\Y, c\neq y} \left(\gc{\xbf}{c}\right).
	\end{align*}
\end{definition}
\vspace{-2mm}
Like in the binary classification framework presented in Section~\ref{sec:cborne}, the majority vote $\BQ(\cdot)$ correctly classifies an example if its $\posterior$-margin is strictly positive, \emph{i.e.,} \ $\RP(\BQ) = \Prob_{\ex\sim D}  \left(\MQ\ex \leq 0\right)$.

We then consider the relaxation proposed by Breiman~\cite{randomforests} that is based on a notion of \emph{strength} of the majority vote in regard to a class $c$.
\begin{definition}[the strength]
	\label{def:force}
	Let $\Hcal$ be a set of multiclass voters from $X$ to $Y$ and let $\posterior$ be a distribution on $\Hcal$. Let $\SQP$ be the random variable that, given an example $\ex$ drawn according to a distribution $D$ over $X\times Y$, outputs the strength of the majority vote $\BQ(\cdot)$ on that example according to a class $c\in Y$, defined by \ $\SQ(\xbf, y) \eqdef  \gc{\xbf}{y} - \gc{\xb}{c}$.
\end{definition} 
\noindent From this definition, one can show that
\vspace{-4mm}
\begin{align}
\RP(\BQ) =&\ \Prob_{\ex\sim D}  \left(\MQ\ex \leq 0\right) 
\ \leq \
\label{eq:aaa}  
\sum_{c=1}^Q \Prob_{(\xb,y)\sim D} \left(\SQ(\xbf, y)  \leq 0 \right)  - 1\,,
\end{align}
where we have the equality in the binary classification setting. Lastly, we consider a relaxation of the notion of margin, that we call the \emph{$\omega$-margin}.
\begin{definition}[the $\omega$-margin]
	\label{def:gammaloss}
	Let $\Hcal$ be a set of multiclass voters from $X$ to $Y$, let $\posterior$ be a distribution on $\Hcal$ and let $\omega \leq 1$. Let $\OQP$ be the random variable that, given an example $\ex\sim D$ over $X\times Y$, outputs the $\omega$-margin of the majority vote $\BQ(\cdot)$ on that example, defined by
	\begin{equation}
	\label{eq:gammaloss}
	\OQ(\xbf, y) \eqdef \gc{\xbf}{y} - 1/\omega\,.
	\end{equation}
\end{definition}

This notion of margin can be seen as the difference between the weight given by the majority vote to the correct class $y$ and a certain threshold $1/ \omega$. In the case of the binary classification, we have that the sign of the $\omega$-margin with $\omega = 2$ is the same than the sign of the binary margin. This observation comes from the fact that $\gc{\xbf}{y}$ is the proportion of voters that vote $y$. In the binary case, this proportion is $\leq \frac{1}{2}$ when the majority vote makes a mistake, and $> \frac{1}{2}$ otherwise. The following theorem relates the risk of $\BQ(\cdot)$ and the  $\omega$-margin associated to $\posterior$.
\begin{theorem}
	\label{theo:link}
	Let $Q\geq 2$ be the number of classes. For every distribution $D$ over $X\times Y$ and for every distribution $\posterior$ over a set of multiclass voters $\Hcal$, we have
	\begin{equation}
	\label{eq:link}
	\Prob_{\ex\sim D}\left(\OQQ\ex \leq 0\right)\quad \leq\quad  \RP(\BQ) \quad  \leq\quad \Prob_{\ex\sim D}\left(\OQtwo\ex \leq 0\right).
	\end{equation}
\end{theorem}
\begin{proof}
	First, let us prove the left-hand side inequality. We have
	\begin{small}
	\begin{align*}
	\RP(B_\posterior)  &= \Prob_{(\xbf,y)\sim D} \left(\MQ\ex \leq 0\right)
	= \Prob_{(\xbf,y)\sim D} \left( \gc{\xbf}{y} \leq \max_{c\in\Y, c\ne y} \gc{\xbf}{c}  \right)\\
	&\geq \Prob_{(\xbf,y)\sim D} \left( \gc{\xbf}{y} \leq \Esp_{c\in\Y, c\ne y} \gc{\xbf}{c}  \right)\\
    &=  \Prob_{(\xbf,y)\sim D} \left( \gc{\xbf}{y} \leq \frac{1}{Q-1}\left[1-\gc{\xbf}{y}\right]  \right)
	=  \Prob_{(\xbf,y)\sim D} \left( \OQQ\ex \leq 0\right)\,.
	\end{align*}
	\end{small}%
The right-hand side inequality is easily verified by observing that the majority vote necessarily makes a correct prediction if the weight given to the correct class $y$ is higher than $\frac{1}{2}$.
\end{proof}

All the above-mentioned notions of margin are equivalent if we stand in the binary classification setting.
However, they differ in the multiclass setting. The multiclass margin of Definition~\ref{def:posterior_marge_multiclasse} is associated to the true decision function in multiclass classification, and is calculated considering all other classes. The strength of Definition~\ref{def:force} depends on the true class $y$ of $\xbf$ and corresponds to a combination of binary margins (one class versus another class) for $c\ne y$. The $\omega$-margin of Definition~\ref{def:gammaloss} also depends on the true class $y$ of $\xbf$, but does not consider the other classes. This measure is easier to manipulate, but implies a higher indecision region (see Theorem~\ref{theo:link}).

\subsection{Generalizations of the \Cbound in the Multiclass Setting} 
The following bound is based on the definition of the multiclass margin in multiclass (Definition~\ref{def:posterior_marge_multiclasse}).
\begin{theorem}[the multiclass \Cbound]
	\label{theo:multibayes_C-bound} 
	For every distribution $\posterior$ on a set of multiclass voters $\Hcal$, and for every distribution $D$ on $X\times Y$, such that $\momentone(\MQP) > 0$, we have
	\begin{equation*}
	\label{eq:multibayes_theo-C-bound}
	\RP(\BQ)\ \leq\ 1 -  \frac{\left(\momentone(\MQP)\right)^2}{\momenttwo(\MQP)}\,.
	\end{equation*}
\end{theorem}
\begin{proof}The proof  is the same than the one of the binary \Cbound (see Theorem~\ref{theo:Cbound_bin}), by considering the multiclass majority-vote of Equation~\eqref{eq:bayes_multiclasse} and the multiclass margin of Definition~\ref{def:posterior_marge_multiclasse}.\end{proof}

This bound offers an accurate relation between the risk of the majority vote and the margin. However, the $\max$ term in the definition of the multiclass margin makes the derivation of an algorithm to minimize this bound much harder than in binary classification.

Thanks to the definition of the strength of Definition~\ref{def:force} and according to the proof process of the \Cbound, we obtain the following relation.
\begin{theorem}
	\label{theo:F-bound} 
	For every distribution $\posterior$ on a set of multiclass voters $\Hcal$, and for every distribution $D$ over $X\times Y$, such that $\forall c\in Y, \ \momentone(\SQP) > 0 $, we have
	\begin{equation*}
	\RP(\BQ)\quad \leq\quad \sum_{c=1}^Q \Prob_{(\xb,y)\sim D} \left( \SQ\ex  \leq 0 \right)-1\quad =\quad (Q - 1) - \sum_{c=1}^{Q}\frac{\left( \momentone(\SQP)\right)^2}{ \momenttwo(\SQP)}\,,
	\end{equation*}
\end{theorem}
\begin{proof}
	The result is obtained by using Inequality~\eqref{eq:aaa} in the proof of the \Cbound.
\end{proof}

This result can be seen as a sum of \Cbounds for every class. A practical drawbacks of this bound in order to construct a minimization algorithm is that we have to minimize a sum of ratios. Finally, the \Cbound obtained by using the $\omega$-margin is given by the following theorem.
\begin{theorem}\label{th:marge_omega}
	For every distribution $\posterior$ on a set of multiclass voters $\Hcal$, for every $\omega \geq 1$, and for every distribution $D$ on $X\times Y$, if  $\momentone(\OQP) > 0$, we have 
	\begin{equation*}
	\Esp_{\ex \sim D} \I\Big(\OQ\ex \leq 0\Big)  \leq\ 1 - \frac{\left(\momentone(\OQP)\right)^2}{\momenttwo(\OQP)}\,.
	\end{equation*}
\end{theorem}
\begin{proof} The result is obtained with the same proof process than the \Cbound, by replacing the use of the random variable $\MQP$ by $\OQP$.
\end{proof}

The $\omega$-margin being linear, we are now able to build a bound minimization algorithm as in Laviolette et al.~\cite{MinCq} for the multiclass classification setting.

\section{Extending the $\omega$-margin to the Multi-label Setting}
\label{sec:cborne_multilabel}

In this section, we will extend the $\omega$-margin with $\omega = 2$ to the more general \emph{multi-label} classification setting. Doing so, we will be able to upper bound the risk of the multi-label majority vote classifier. We stand in the multi-label classification setting where the input space is still $X\! \subseteq\! \R^d$, the space of possible labels is $Y\! =\! \{1,\ldots,Q\}$ with a finite number of classes $Q\!\geq\! 2$, but we consider the output space $\overline{Y} = \{0, 1\}^Q$ that contains vectors $\overline{y}$ of length $Q$ where the $i^{\mbox{\tiny th}}$ element is $1$ if example~$i$ is among the labels associated to the example $x$, and $0$ otherwise. We consider a set $\overline{\Hcal}$ of \emph{multi-label voters} $\overline{h} : X \mapsto \overline{Y}$ .  As usual in structured output prediction, given a distribution $\posterior$ over $\overline{\Hcal}$, the multi-label majority vote classifier $\overline{B_\posterior}$ chooses the label $\overline{c}\in\overline{Y}$ that has the lowest squared Euclidean distance with the $\posterior$-weighted  cumulative confidence,
\begin{align*}
\overline{B_\posterior}(x)  \ \eqdef\  \argmin_{\overline{c}\in \overline{Y}}\left\|\overline{c} - \Esp_{\overline{h}\sim \posterior} \overline{h}(x)\right\|^2\ =\ \argmax_{\overline{c} \in \overline{Y}} \left[  \overline{c} \cdot \left( \Esp_{\overline{h}\sim \posterior} \overline{h}(x) - \frac{1}{2}\,\textbf{1} \right) \right]\,,
\end{align*}
where $\textbf{1}$ is a vector of length $Q$ containing ones. The multi-label margin is given by Definition~\ref{def:posterior_marge_multilabel}.
\begin{definition}[the multi-label margin]\rm
	\label{def:posterior_marge_multilabel}
	Let $D$ be a distribution over $X\times \overline{Y}$, let $\overline{\Hcal}$ be a set of multi-label voters. Given a distribution $\posterior$ on $\overline{\Hcal}$, the margin of the majority vote $\overline{\BQ}(\cdot)$ on  $(\xbf,\overline{y})$ is
	\begin{align*}
	\overline{\MQ}(\xbf,\overline{y})  \eqdef \left( \Esp_{\overline{h}\sim \posterior} \overline{h}(x) - \frac{1}{2}\,\textbf{1} \right) \cdot \left(\overline{y} - \frac{1}{2}\,\textbf{1}  \right)  -  \max_{\overline{c}\in\overline{Y}, \overline{c}\neq \overline{y}} \left( \Esp_{\overline{h}\sim \posterior} \overline{h}(x) - \frac{1}{2}\,\textbf{1} \right) \cdot \left(\overline{c} - \frac{1}{2}\,\textbf{1}  \right) \,.
	\end{align*}
\end{definition}

As we did in the multiclass setting with the margin of Definition~\ref{def:posterior_marge_multiclasse}, we can upper bound the risk of the multi-label majority vote classifier by developing a \Cbound using the margin of Definition~\ref{def:posterior_marge_multilabel}. However, as this margin also depends on a $\max$ term, the derivation of a learning algorithm minimizing the resulting \Cbound remains hard. To overcome this, we generalize $\OQtwo$, the $\omega$-margin with $\omega=2$  of Definition~\ref{def:gammaloss}, to the multi-label setting, as follows.
\begin{definition}\rm
	\label{def:posterior_omega_marge_multilabel}
	Let $D$ be a distribution over $X\times \overline{Y}$, let $\overline{\Hcal}$ be a set of multi-label voters. Given a distribution $\posterior$ on $\overline{\Hcal}$, the $2$-margin of the majority vote $\BQ(\cdot)$ on  $(\xbf,\overline{y})$ is
	\begin{eqnarray*}
	\overline{M_{\rho, 2}}(\xbf,\overline{y}) & \eqdef  & \left(\Big(\Esp_{\overline{h}\sim \posterior} \overline{h}(x) - \frac{1}{2}\, \textbf{1}\Big)-\Big(\overline{y}_{i\rightarrow {1}/{2}} -\frac{1}{2}\, \textbf{1}\Big)\right) \cdot \Big(\overline{y} -\frac{1}{2}\, \textbf{1}\Big)\\
	& \eqdef  &
	\overline{y} \cdot \left(\Esp_{\overline{h}\sim \posterior} \overline{h}(x) - \frac{1}{2}\, \textbf{1}\right) - \Esp_{\overline{h}\sim \posterior} \overline{h}(x) \cdot \frac{1}{2}\,\textbf{1} - \frac{1}{4}\,,
	\end{eqnarray*}
	where $i\in \{1,..,Q\}$ and $\overline{y}_{i\rightarrow {1}/{2}}$ is obtained from $\overline{y}$ by replacing its $i^{\mbox{\tiny th}}$ coordinate by~$1/2$.
\end{definition}
The second equality of the definition is obtained by straightforward calculation.
Now, let $P_{\overline{y}}$, be the only hyperplane on which lies all the points of the form $\overline{y}_{i\rightarrow {1}/{2}}$ for $i = 1,\ldots,Q$. Since this hyperplane has normal $\Big(\overline{y} -\frac{1}{2}\, \textbf{1}\Big)$, it follows from basic linear algebra that if $\overline{M_{\rho, 2}} >0$, then vectors $\Esp_{\overline{h}\sim \posterior} \overline{h}(x)$ and $\overline{y}$ will be on the same side of $P_{\overline{y}}$. It is also easy to see that in this case, we have $\overline{B_\posterior}(x)=\overline{y}$. Figure~\ref{fig:multilabel} shows an example in the case where $Q=2$. Thus, we  have that $ \RP(\BQ) \,  \leq\, \Prob_{\ex\sim D}\left(\overline{M_{\rho, 2}}(\xbf,\overline{y}) \leq 0\right)$, and following the same arguments as in Theorem~\ref{th:marge_omega}, one can derive the following  multi-label \Cbound.

\begin{theorem}
	For every distribution $\posterior$ on a set of multi-label voters $\overline{\Hcal}$ and for every distribution $D$ on $X\times \overline{Y}$, if  $\momentone(\overline{M_{\rho, 2}}(\xbf,\overline{y})) > 0$, we have 
	\begin{equation*}
	\RP(\overline{\BQ}) \quad \leq \quad \Esp_{(x, \overline{y}) \sim D} \I\Big(\overline{M_{\rho,2}}(x, \overline{y}) \leq 0\Big) \quad  \leq\quad 1 - \frac{\left(\momentone(\overline{M_{\rho,2}^D})\right)^2}{\momenttwo(\overline{M_{\rho,2}^D})}\,.
	\end{equation*}
\end{theorem}

\begin{figure}
	\centering
	\includegraphics[width=0.3\linewidth]{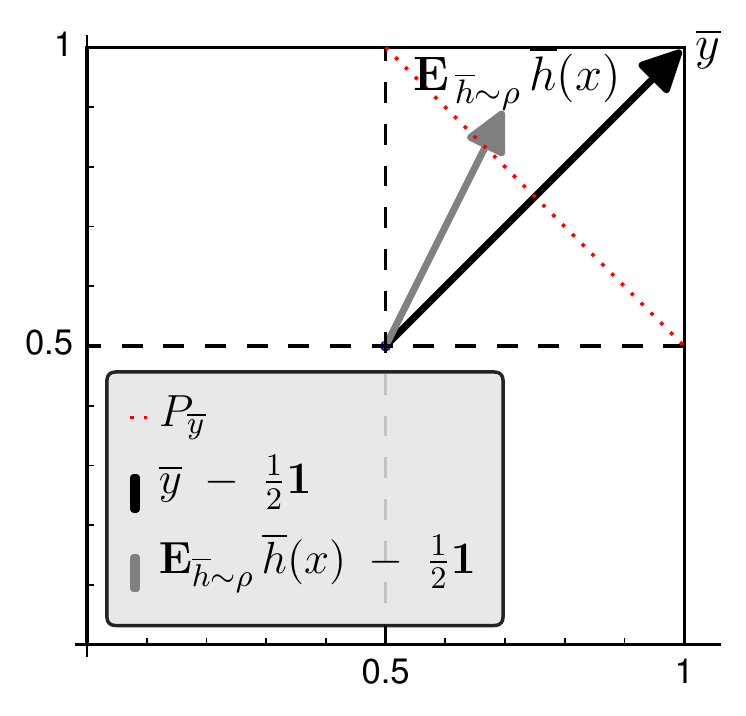}
	\caption{Graphical representation of label $\overline{y}$, hyperplane $P_{\overline{y}}$ and vector $\Esp_{\overline{h} \sim \rho} \overline{h}(x)$.}
	\label{fig:multilabel}
\end{figure}

\section{Conclusion and Outlooks}
\label{sec:conclu}
In this paper, we extend an important theoretical result in the PAC-Bayesian literature to the multiclass and multi-label settings. Concretely, we prove three multiclass versions and one multi-label version of the \Cbound, a bound over the risk of the majority vote, based on generalizations of the notion of margin for multiclass and multi-label classification. These results open the way to extending the theory to more complex outputs and developing new algorithms for multiclass and multi-label classification with PAC-Bayesian generalization guarantees.

\newpage

\bibliography{biblio}
\bibliographystyle{unsrt}

\end{document}